\documentclass[11pt]{article}
\usepackage[utf8]{inputenc}
\usepackage[T1]{fontenc}
\usepackage{amsthm, amsmath}
\usepackage{amsfonts}
\usepackage[margin=1in]{geometry}
\usepackage[algo2e,ruled,noend,resetcount,linesnumbered]{algorithm2e}
\usepackage{graphicx}
\usepackage{comment}
\usepackage[shortlabels]{enumitem}
\usepackage{thmtools}
\usepackage{thm-restate}
\usepackage{algorithm, algorithmicx, algpseudocode} 
\usepackage{array}
\usepackage{stmaryrd}

\usepackage[draft,margin=false,inline=true]{fixme}
\fxsetup{mode=multiuser,theme=color, layout=inline}
\FXRegisterAuthor{bs}{abs}{\color{blue} {\bf Barna}}

\setlist[enumerate]{nosep, topsep=1ex}
\setlist[itemize]{nosep, topsep=1ex}
\setlist[description]{nosep}
\allowdisplaybreaks
\usepackage[colorlinks,citecolor=blue,linkcolor=blue,urlcolor=red,pagebackref]{hyperref}
\usepackage[capitalise]{cleveref}
\usepackage{xspace}

\usepackage{caption}
\usepackage[indention=20pt]{subcaption}

\newcommand{\alg}[1]{\textsc{\bfseries \footnotesize #1}}
\newcommand{\innerAlg}{\mathcal{A}}
\newcommand{\outerAlg}{\mathcal{B}}

\newcommand{\bigO}[1]{O \left( #1 \right)}
\newcommand{\littleO}[1]{o \left( #1 \right)}
\newcommand{\bigOmega}[1]{\Omega \left( #1 \right)}

\newcommand{\bigTheta}[1]{\Theta \left( #1 \right)}

\newtheorem{theorem}{Theorem}[section]

\newtheorem{definition}[theorem]{Definition}
\newtheorem{lemma}[theorem]{Lemma}

\newcommand{\R}{\mathbb{R}}

\newcommand{\F}{\mathbb{F}}

\newcommand{\class}{\mathcal{C}}

\newcommand{\domain}{\mathcal{X}}
\newcommand{\range}{\mathcal{Y}}

\newcommand{\partition}{\mathcal{P}}
\newcommand{\mechanism}{\mathcal{M}}

\newcommand{\given}{\textrm{\xspace s.t. \xspace}}

\newcommand{\ceil}[1]{\left\lceil #1 \right\rceil}
\newcommand{\floor}[1]{\left\lfloor #1 \right\rfloor}
\newcommand{\set}[1]{\{#1\}}

\newcommand{\wt}{\mathsf{wt}}

\newcommand{\diag}{\mathrm{diag}}

\newcommand{\numberthis}{\addtocounter{equation}{1}\tag{\theequation}}
\newcommand{\squareTiling}{\alg{SquareTilingAttention}}

\newcommand{\matrixEntryCompression}{\mathbf{MatrixEntryCompression}}

\newcommand{\oneWayCC}{\mathbf{OneWayCC}}

\newcommand{\bch}[2]{\mathrm{BCH}[#1, #2]}

\newcommand{\softmax}{\mathrm{softmax}}

\title{The I/O Complexity of Attention, or \\ How Optimal is Flash Attention?}

\author{Barna Saha\thanks{University of California, San Diego. The authors are partially supported by NSF grants 1652303, 1909046, 2112533, and HDR TRIPODS Phase II grant 2217058.} \and Christopher Ye\footnotemark[1]}

\date{\today}

\begin{document}

\maketitle

\begin{abstract} \small\baselineskip=9pt
    Self-attention is at the heart of the popular Transformer architecture, yet suffers from quadratic time and memory complexity. 
In a recent significant development, FlashAttention shows that the I/O complexity of attention is the true bottleneck in scaling Transformers.  
Given two levels of memory hierarchy, a fast cache (e.g. GPU on-chip SRAM) where computation happens and a slow memory (e.g. GPU high-bandwidth memory) where the data resides, the I/O complexity measures the number of accesses to the slow memory. 
FlashAttention is an I/O-aware algorithm for self-attention that requires $\frac{N^2d^2}{M}$ I/O operations where $N$ is the dimension of the attention matrix, $d$ is the head-dimension and $M$ is the size of cache. 
{\it However, is this I/O complexity optimal?} The known lower bound only rules out an I/O complexity of $o(Nd)$ when $M=\Theta(Nd)$, since the output of the attention mechanism that needs to be written in the slow memory is $\Omega(Nd)$. 
The main question that remained open after FlashAttention is whether this I/O complexity is optimal for any value of M.

We resolve the above question in its full generality by showing an I/O complexity lower bound that matches the upper bound provided by FlashAttention for any values of $M \geq d^2$ within any constant factors. 
Further, we give a better algorithm with lower I/O complexity for $M < d^2$, and show that it is optimal as well. 
Moreover, our lower bounds do not rely on using combinatorial matrix multiplication for computing the attention matrix. 
We show even if one uses fast matrix multiplication, the above I/O complexity bounds cannot be improved. 
We do so by introducing a new communication complexity protocol for matrix compression, and connecting communication complexity to I/O complexity. 
To the best of our knowledge, this is the first work to establish a connection between communication complexity and I/O complexity, and we believe this connection could be of independent interest and will find many more applications in proving I/O complexity lower bounds in future.

\end{abstract}

\newpage

\section{Introduction}
Transformer models \cite{DBLP:conf/nips/VaswaniSPUJGKP17} have emerged as the architecture of choice for a variety of applications including natural language processing and computer vision. 
The self-attention module at the heart of the Transformer architecture requires quadratic time and memory complexity.
Thus despite its popularity, Transformers are slow and memory-hungry.  
In a seminal paper, \cite{DBLP:conf/nips/DaoFERR22} proposes FlashAttention, an IO-aware algorithm for self-attention.
Given two levels of memory hierarchy, a fast cache (e.g. GPU on-chip SRAM) where computation happens and a slow memory (e.g. GPU high-bandwidth memory) where the data resides, the I/O complexity measures the number of accesses to the slow memory. 
\cite{DBLP:conf/nips/DaoFERR22} argues that I/O complexity is indeed the true bottleneck in achieving wall-clock speed up, and provide an algorithm that achieves an I/O complexity of $O(\frac{N^2d^2}{M})$ where $N$ is the dimension of the attention matrix, $d$ is the head-dimension and $M$ is the cache size. 
In self-attention, three input matrices $Q, K, V \in \R^{N \times d}$ are used to compute $\softmax(QK^T) V$ where $A = \exp(QK^T)$ is the attention matrix. 
For $M=\Theta(Nd)$, the above $I/O$ complexity bound becomes $Nd$ which is also needed just to write the output of self-attention to slow memory. 
This leaves a tantalizing open question whether FlashAttention I/O complexity is optimal for any values of $M$. 

We resolve this question in affirmative for any value of $M$ as long as $M \geq d^2$. 
Moreover, for $M < d^2$, we give a better algorithm than FlashAttention and also show the bound is optimal. 
Indeed as we prove for $M \leq d^2$, the I/O complexity of self-attention is at least as high as the I/O complexity of computing matrix multiplication of two $N \times d$ and $d \times N$ matrices. 
For $M \geq d^2$ which is also the more practical regime, our proof of optimality of FlashAttention brings in several new ideas.

First, we show that FlashAttention which uses combinatorial matrix multiplication has optimal I/O complexity by utilizing the famous {\it red-blue pebble game} from the early eighties \cite{redblue1981}. 
However, it may still be possible to improve FlashAttention by utilizing fast matrix multiplication \cite{DBLP:conf/soda/GallU18, DBLP:journals/corr/FMM} to obtain lower I/O complexity. 
We rule out any such possibilities. We prove a general lower bound against all such algorithms. 
We introduce a new matrix compression problem, and show that it has a high communication complexity. 
Next, we establish a connection between the communication complexity of a matrix compression protocol with the I/O complexity of attention. 
To the best of our knowledge, this is the first work that shows a connection between communication complexity and I/O complexity, which may have further valuable theoretical consequences.
We show even when the input matrix entries are restricted to binary, the lower bound holds within polylogarithmic factors by utilizing Binary BCH codes \cite{hocquenghem1959codes, DBLP:journals/iandc/BoseR60a} from coding theory to prove lower bounds against binary matrix compression.




\subsection{Related Work}

Many papers have studied attention, a key bottleneck of the transformer architecture \cite{DBLP:conf/nips/VaswaniSPUJGKP17} which has become the predominant architecture in applications such as natural language processing.
Many approximate notions have been studied to reduce its compute and memory constraints \cite{DBLP:conf/nips/BrownMRSKDNSSAA20,   DBLP:conf/icml/KatharopoulosV020, DBLP:conf/iclr/KitaevKL20,  DBLP:conf/nips/ZaheerGDAAOPRWY20, DBLP:conf/nips/ChenDWSRR21, DBLP:conf/iclr/ChoromanskiLDSG21, DBLP:journals/corr/BoundedEntryAttention, DBLP:journals/corr/HyperAttention}.
Algorithms such as FlashAttention have made significant practical gains by designing I/O-aware algorithms \cite{DBLP:conf/nips/DaoFERR22, DBLP:journals/corr/abs-2307-08691}.

The role of learning under bounded space constraints has been studied primarily in the context of bounded memory, with applications including Online Learning \cite{DBLP:conf/stoc/SrinivasWXZ22, DBLP:conf/focs/PengR23, DBLP:conf/soda/PengZ23}, continual learning \cite{DBLP:conf/focs/ChenPP22}, convex optimization \cite{DBLP:conf/colt/WoodworthS19, DBLP:conf/colt/MarsdenSSV22, DBLP:conf/focs/ChenP23}, and others \cite{DBLP:journals/jacm/Raz19, DBLP:conf/stoc/SharanSV19, DBLP:conf/nips/GonenLM20}.
In this work, we instead consider learning in the context of a \emph{two-level} memory hierarchy with a \emph{bounded cache}.
While the algorithm has unlimited memory, it can only access a small portion of this memory at any given time, and is charged for every access to memory.
We believe that the task of minimizing I/O complexity is both theoretically interesting and practically significant, as I/O operations often form a key computational bottleneck in practice.

There is a long-line of work on I/O complexity \cite{redblue1981, DBLP:journals/cacm/AggarwalV88, DBLP:conf/pods/PaghS14}. 
Within this body of work, many papers have focused on the problem of matrix multiplication \cite{DBLP:journals/jacm/BallardDHS12, DBLP:conf/medalg/BallardDHLS12, DBLP:conf/esa/PaghS14, DBLP:conf/spaa/ScottHS15, DBLP:conf/wads/BilardiS17}.

\subsection{Our Contributions}

We study the I/O complexity of attention on a two-level memory hierarchy.
The attention mechanism takes inputs $Q, K, V \in \R^{N \times d}$ and computes $\softmax(QK^T) V$ (see Section \ref{sec:prelims:attention}) where $A = \exp(QK^T)$ is the attention matrix.
Our first result resolves the I/O complexity of any algorithm computing attention using the standard matrix multiplication algorithm.
This answers an open question of \cite{DBLP:conf/nips/DaoFERR22} and shows that FlashAttention is optimal among this class of algorithms.

\begin{restatable*}{theorem}{standardMMAttention}
    \label{thm:standard-mm-attention-mechanism-i/o}
    The I/O complexity of attention with standard matrix multiplication is \begin{equation*}
        \bigTheta{\min\left( \frac{N^2 d}{\sqrt{M}}, \frac{N^2 d^2}{M} \right)}
    \end{equation*}
\end{restatable*}

We divide the analysis into two cases, the large cache ($M \geq d^2$) and small cache ($M < d^2$) case.
When $M < d^2$, we show that attention and matrix multiplication are equivalent. 
In the small cache setting, $\frac{N^2 d^2}{M} > \frac{N^2 d}{\sqrt{M}} > N^2$, so we can explicitly write $QK^T$ to memory while computing attention.
Thus, any algorithm computing attention in fact computes $QK^T$, establishing an equivalence between attention and matrix multiplication.

In the more interesting large cache case, when $M \geq d^2$, the upper bound is given by the breakthrough FlashAttention algorithm \cite{DBLP:conf/nips/DaoFERR22}.
In this setting, the I/O complexity approaches $O(N d)$ as cache size increases, providing significant practical improvements, despite the theoretical time complexity remaining $O(N^2 d)$.
To prove a matching lower bound, we use the red-blue pebble game of \cite{redblue1981}.
Executing an algorithm $\innerAlg$ on a machine with cache size $M$ is equivalent to playing the red-blue pebble game on the computational directed acyclic graph $G$ corresponding to $\innerAlg$.
Specifically, the partitioning lemma (Lemma \ref{lemma:m-partition-i/o-lb}) states that any successful execution of $\innerAlg$ corresponds to a partition of $G$ where each part corresponds to a sub-computation of $\innerAlg$ with no I/O operations (called a $M$-partition).
Thus, it suffices to prove a lower bound on the size of any $M$-partition to obtain a I/O complexity lower bound for $\innerAlg$ on a machine with cache size $M$.
Our lower bound on the partition size then follows from a careful analysis of the computational graph of attention (Figure \ref{fig:flash-attention-2-graph}).
Specifically, we show that any part $V_i$ in the partition can only compute $\frac{M^2}{d^2} + M \leq \frac{2 M^2}{d^2}$ entries in the product $QK^T$.

We then extend our lower bound in the large cache ($M \geq d^2$) case to arbitrary algorithms for attention.
More precisely, we show that even with fast matrix multiplication (FMM), $\bigOmega{\frac{N^2 d^2}{M}}$ I/O operations are required to compute attention.

\begin{restatable*}{theorem}{largeFieldAttention}
    \label{thm:attention-i/o-lb}
    Suppose $Q, K \in \F_q^{N \times d}$ where $q > N$.
    The I/O complexity of attention (with any matrix multiplication algorithm) is $\bigOmega{\min\left(\frac{N^2 d^2}{M}, N^2 \right)}$.
\end{restatable*}

To establish a general lower bound, we can no longer reason about a specific computational graph.
Instead, we appeal to communication complexity, a useful tool for proving lower bounds for learning \cite{DBLP:conf/colt/DaganKS19, DBLP:conf/colt/KaneLMY19}, specifically space-bounded learning \cite{DBLP:conf/focs/ChenPP22}.
To the best of our knowledge, this is the first application of communication complexity to I/O complexity lower bounds.

For simplicity, assume that $\innerAlg$ performs I/O operations in batches of size $M$ (Lemma \ref{lemma:execution-epoch-partition} shows that this can be assumed without loss of generality).
Within each batch, $\innerAlg$ can only access the cache of size $M$, with no further I/O operations.
In order to give a lower bound on the number of batches, we argue that $\innerAlg$ cannot make too much progress in each batch.
In the context of attention, we measure progress as the number of entries computed in $QK^T$ (regardless of whether these entries are written to memory).

Formally, we introduce the $B$-entry matrix compression problem (Definition \ref{def:matrix-entry-compression}) and argue that any algorithm making too much progress on a given batch gives an efficient communication protocol (Theorem \ref{thm:i/o-compression-protocol}).
We then complete the argument by giving a lower bound on the communication complexity of the matrix compression problem (Lemma \ref{lemma:d-rank-cache-lb}).
Our lower bound follows from constructing input matrices satisfying strong linear independence constraints.
For large $q$, this is achieved by Vandermonde matrices.

Finally, we relax the constraint $q > N$ and consider the special case $q = 2$, i.e. $Q, K, V$ are binary matrices and obtain a lower bound tight up to polylogarithmic factors.
Our lower bound uses Binary BCH codes \cite{hocquenghem1959codes, DBLP:journals/iandc/BoseR60a} to construct matrices satisfying strong linear independence constraints.

\begin{restatable*}{theorem}{binaryAttention}
    \label{thm:attention-i/o-lb-binary}
    Suppose $Q, K \in \set{0, 1}^{N \times d}$.
    The I/O complexity of attention (with any matrix multiplication algorithm) is $\bigOmega{\min\left(\frac{N^2 d^2}{M \log^2 N}, N^2 \right)}$.
\end{restatable*}

In the small cache setting ($M < d^2$), Theorem \ref{thm:small-cache-fmm-equiv} gives a simple equivalence between attention and rectangular matrix multiplication.

\section{Preliminaries}

Given a matrix $A \in \R^{n \times m}$, we index an individual entry as $A[i, j]$. 
The $i$-th row is denoted $A[i]$ while the $j$-th column is denoted $A[*, j]$.
$A[i_1:i_2, j_1:j_2]$ denotes a block of $A$ consisting of entries $(i, j)$ where $i \in [i_1, i_2]$ and $j \in [j_1, j_2]$.
Given a block size $B$, the block $A[(i - 1) \cdot B + 1: i \cdot B, (j - 1) \cdot B + 1:j \cdot B]$ is denoted $A^{(B)}[i, j]$.

$A^T, A^{-1}$ denote the transpose and inverse of $A$.

For a vector $v \in \R^{n}$, we similarly denote entries $v[i]$, a contiguous block of entries as $v[i_1:i_2]$, and the $i$-th block of size $B$ as $v^{(B)}[i]$.
Let $\diag(v)$ denote the matrix $D \in \R^{n \times n}$ with $D[i, i] = v[i]$.

\subsection{The Attention Mechanism}
\label{sec:prelims:attention}

Given input matrices $Q, K, V \in \R^{N \times d}$, the attention mechanism computes $D^{-1} A V$ where the attention matrix $A = \exp(QK^T)$ is computed by taking exponents entry-wise and $D = \diag(A \cdot \mathbf{1})$ is the diagonal matrix containing row-sums of $A$.

\subsection{The Memory Hierarchy}

In this work, we assume that the memory hierarchy has two levels: the small but fast layer (called the \emph{cache}) and the large slow layer (called the \emph{memory}).
We assume that the memory is unbounded while the cache is bounded by some size constraint $M$.
Furthermore, computations only occur on the cache.

\subsection{Red-Blue Pebble Game}

We require the red-blue pebble game of \cite{redblue1981}, designed to model computations on a two-level memory hierarchy.
Inspired by the pebble game often used to model space-bounded computation, \cite{redblue1981} develops the red-blue pebble game to model I/O complexity.
As with the standard pebble game, the red-blue pebble game is played on a directed acyclic graph, where nodes represent computations and edges represent dependencies.
Throughout the game, pebbles can be added to, removed from, or recolored between red and blue in the graph, where red pebbles represent data in cache and blue pebbles represent data in slow memory.
Given an upper bound on the number of red pebbles (cache size), the goal is to minimize the number of pebble recolorings (I/O operations) over the course of a computation.
The game is defined formally below.

\begin{definition}
    [Red-Blue Pebble Game~\cite{redblue1981}]
    Let $G$ be a directed acyclic graph with a set of \emph{input} vertices containing at least all vertices with no parents, and a set of \emph{output} vertices containing at least all vertices with no children.
    A configuration is a pair of subsets of vertices, one containing all vertices with red pebbles, and the other containing all vertices with blue pebbles.
    Note that a vertex can have a red pebble, a blue pebble, both, or neither.

    The \emph{initial} (resp. \emph{terminal}) configuration is one in which only input (resp. output) vertices contain pebbles, and all of these pebbles are blue.
    The rules of the red-blue pebble game are as follows:
    \begin{enumerate}[start=1,label={\bfseries R\arabic*}]
        \item (Input) A red pebble may be placed on any vertex with a blue pebble.
        \label{rb-pebble-rule:input}
        \item (Output) A blue pebble may be placed on any vertex with a red pebble.
        \label{rb-pebble-rule:output}
        \item (Compute) A red pebble may be placed on a vertex if all its parents have red pebbles.
        \label{rb-pebble-rule:compute}
        \item (Delete) A pebble may be removed from any vertex.
        \label{rb-pebble-rule:delete}
    \end{enumerate}
    A \emph{transition} is an ordered pair of configurations where the second can be obtained from the first following one of the above rules.
    A \emph{calculation} is a sequence of configurations, where each successive pair forms a transition.
    A \emph{complete calculation} is a calculation that begins with the initial configuration and ends with the terminal configuration.
\end{definition}

In typical instances, including our paper, the input and output vertices are disjoint. 
Furthermore, the input vertices are typically exactly those with no parents, while the output vertices are exactly those with no children.
To model a bounded cache, we assume that there are at most $M$ red pebbles on the graph at any given time, while any number of blue pebbles can be placed on the graph.
Given a graph $G$, we are interested in the I/O complexity.

\begin{definition}[I/O Complexity~\cite{redblue1981}]
    \label{def:io-complexity}
    Given a graph $G$ and integer $M$, the I/O complexity $Q(G, M)$ is defined by the minimum number of transitions according to rules \ref{rb-pebble-rule:input} and \ref{rb-pebble-rule:output} required by any complete calculation.
    When the underlying graph $G$ is clear, we omit $G$ and write $Q(M)$.
\end{definition}

To obtain I/O complexity lower bounds, \cite{redblue1981} show that any complete calculation induces a $M$-partition.
Roughly speaking, a $M$-partition partitions the vertex set such that each part can be computed completely using only $2M$ I/O operations.
A lower bound on the size of any $M$-partition then implies a lower bound on the I/O complexity of the computational graph.

Before defining the $M$-partition, we give the necessary definitions of dominator sets, minimum sets, and vertex subset dependence.

\begin{definition}[Dominator Set]
    Let $G$ be a directed acylic graph and $S \subset V$ a subset.
    $D \subset V$ is a \emph{dominator set} of $S$ if for every path from an input vertex of $G$ to a vertex in $S$ contains a vertex in $D$.
    \label{def:dominator-set}
\end{definition}

\begin{definition}[Minimum Set]
    Let $G$ be a directed acylic graph and $S \subset V$ a subset.
    The \emph{minimum set} of $S$, denoted $M$, is the subset of vertices in $S$ with no children in $S$.
    \label{def:minimum-set}
\end{definition}

\begin{definition}[Vertex Subset Dependence]
    \label{def:cyclic-dependence}
    Let $G$ be a directed acyclic graph, with $S, T \subset V$ disjoint subsets.
    $T$ \emph{depends} on $S$ if there exists an edge $(s, t) \in E$ with $s \in S, t \in T$.
\end{definition}

We now give the definition of an $M$-partition (called an $S$-partition in \cite{redblue1981}).

\begin{definition}[$M$-partition~\cite{redblue1981}]
    \label{def:m-partition}
    Let $G$ be a directed acyclic graph. 
    A family of subset $\set{V_i}_{i = 1}^{h}$ is a \emph{$M$-partition} of $G$ if the following properties hold:
    \begin{enumerate}[start=1,label={\bfseries P\arabic*}]
        \item (Partition) The sets $V_i$ are disjoint and $V = \bigcup_{i = 1}^{h} V_i$.
        \label{m-partition:partition}
        \item (Dominator) For each $V_i$, there exists a dominator set $D_i$ of size at most $M$.
        \label{m-partition:dominator}
        \item (Minimum) For each $V_i$, the set of minimum vertices $M_i$ has size at most $M$.
        \label{m-partition:minimum}
        \item (Acyclic) There is no cyclic dependence among vertex sets $\set{V_i}_{i = 1}^{h}$.
        \label{m-partition:acyclic}
    \end{enumerate}
\end{definition}

Roughly speaking, a $M$-partition partitions the vertex set such that each part can be computed completely using only $2M$ I/O operations.
In particular, every node in a single part $V'$ can be reached by placing red pebbles at the dominator vertices and using Rule \ref{rb-pebble-rule:compute}.
This can be thought of as the initial state of the cache at some point, and without any further reads from memory, every node on $V'$ can be computed.
Then, after the computation, the values at the minimum vertices can be written to memory.
This corresponds to at most $M$ write operations.
The following lemma relates I/O complexity to the size of $M$-partitions.

\begin{lemma}[Theorem 3.1 of \cite{redblue1981}]
    \label{lemma:m-partition-i/o-lb}
    Let $G$ be a directed acyclic graph.
    Any complete calculation of the red-blue pebble game on $G$, using at most $M$ red pebbles, is associated with a $2M$-partition of $G$ such that,
    \begin{equation*}
        M \cdot h \geq Q(G, M) \geq M \cdot (h - 1)
    \end{equation*}
    where $h$ is the number of vertex subsets in the $2M$-partition.
\end{lemma}

In particular, the partitioning lemma of \cite{redblue1981} shows that it suffices to prove a lower bound on the size of any $2M$-partition to obtain a lower bound on the I/O complexity.

\begin{lemma}[Lemma 3.1 of \cite{redblue1981}]
    \label{lemma:partition-i/o-lb}
    For any directed acyclic graph $G$, let $P(G, M)$ denote the minimum size of any $M$-partition of $G$.
    Then,
    \begin{equation*}
        Q(G, M) \geq M \cdot (P(G, 2M) - 1)
    \end{equation*}
\end{lemma}

\subsection{Computational Graph for the Attention Mechanism}
\label{sec:attention-computational-graph}

Figure \ref{fig:flash-attention-2-graph} gives a computational graph modelling the attention mechanism \cite{DBLP:conf/nips/DaoFERR22, DBLP:journals/corr/abs-2307-08691}.
In this computational graph, we assume that matrix products are computed using the standard algorithm, that is, we compute $(AB)_{ij} = \sum_{k} A_{ik} B_{kj}$ for all $i, j$.

\begin{figure}[ht]
    \centering
    \includegraphics[width=0.9\columnwidth]{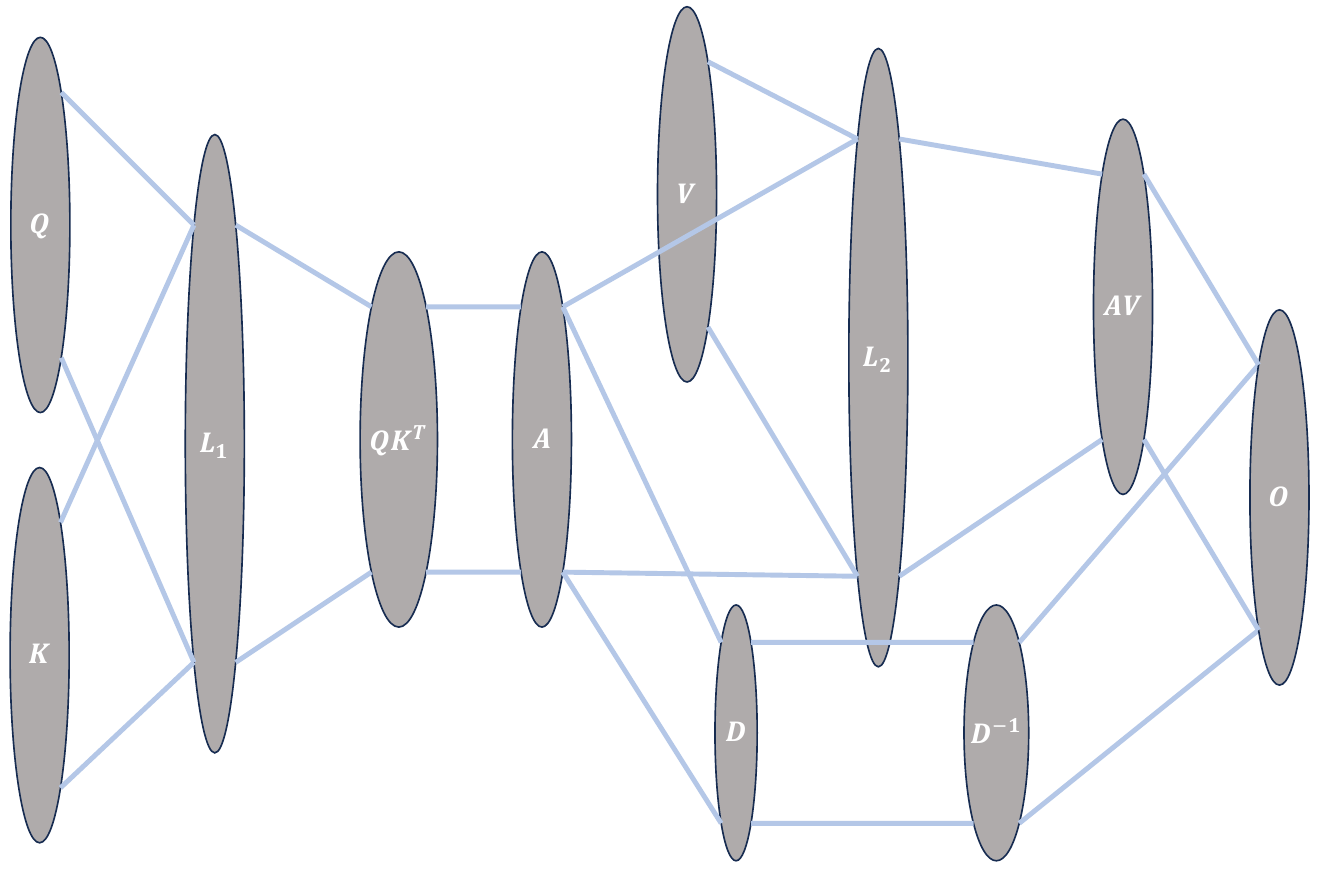}
    \caption{Computational Graph for the Attention Mechanism.}
    \label{fig:flash-attention-2-graph}
\end{figure}

The attention mechanism begins by computing the matrix product $QK^T$.
In the vertex set $L_1$, there are $N^2 d$ vertices representing the values $\set{Q_{i \ell} K^T_{\ell j}}_{i, j, \ell}$.
Then, each entry $(QK^T)_{ij}$ is computed by summing the appropriate vertices in $L_1$, i.e. the node $(QK^T)_{ij} = \sum_{\ell} Q_{i \ell} K^T_{\ell j}$.
In particular, each entry $(QK^T)_{ij}$ is connected to $L_1$ via a summation tree.
Note that all nodes in the summation trees are disjoint.
The summation tree can be thought of as a balanced binary tree with $d$ leaves, although the exact structure of the tree (i.e. order of summation) can be arbitrary. 
For a more detailed description of the computational graph of the standard matrix multiplication algorithm, see \cite{redblue1981}.
We define the set of \emph{level-1 vertices} to be all nodes in the computational graph that are in vertex sets $L_1, QK^T$, or one of the intermediate nodes in the $N^2$ summation trees between the two layers.
Figure \ref{fig:summation-tree} illustrates an example summation tree.

\begin{figure}[ht]
    \centering
    \includegraphics[width=0.6\columnwidth]{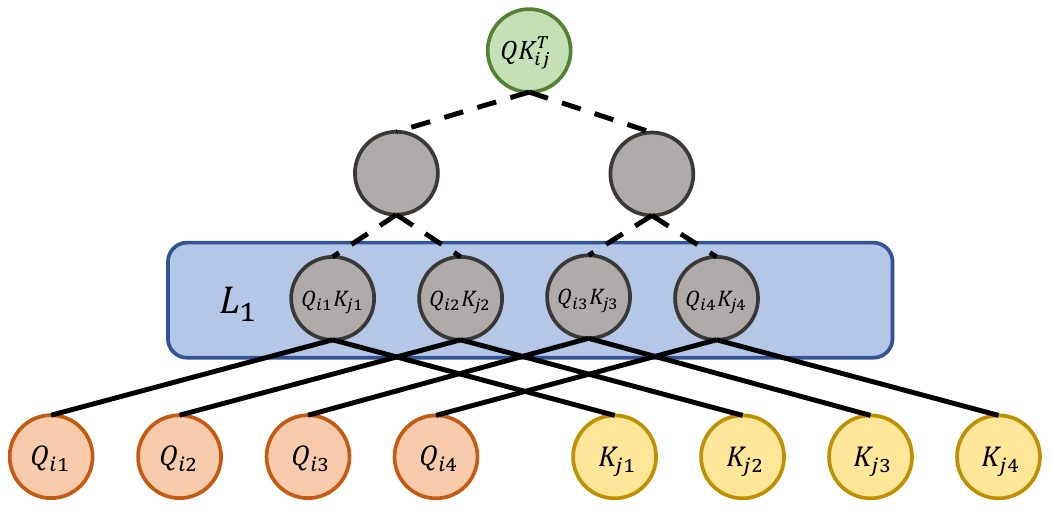}
    \caption{A single summation tree with $d = 4$. 
    Orange vertices denote inputs from $Q$. 
    Yellow vertices denote inputs from $K$.
    Grey and green vertices denote \emph{level-1 vertices}.
    Observe that these are disjoint for each summation tree.
    The green vertex specifically denotes an entry in $QK^T$.
    Solid edges denote multiplications and dotted edges denote additions.
    Vertices in the blue box denote vertices in $L_1$.}
    \label{fig:summation-tree}
\end{figure}

Then, each entry in $A$ is computed by taking the exponent of the corresponding entry in $QK^T$.
Each node in $D$ is computed by summing over the rows of $A$.
As above, this is realized in the graph by $N$ disjoint summation trees and $D^{-1}$ is computed by taking the multiplicative inverse of each element in $D$.
The matrix product is computed as in the first step, first by storing all $N^2 d$ products in the intermediate layer $L_2 = \set{A_{ik}V_{kj}}_{i, k, j}$ and computing $AV$ via $Nd$ disjoint summation trees.
Finally, each node in $O$ is computed by scaling each entry of $AV$ by the appropriate factor in $D^{-1}$.

Although we have (roughly) described the computational graph of FlashAttention-2 \cite{DBLP:journals/corr/abs-2307-08691} and there are different ways to implement attention, all algorithms begin by computing the product $QK^T$, and our lower bounds will hold for any algorithm computing this matrix product.

\section{I/O Complexity of Attention}
\label{sec:standard-mm-i/o-complexity}

In this section, we present a tight characterization of the I/O complexity of any algorithm computing attention exactly using the standard matrix multiplication algorithm.

\standardMMAttention

At the crossover point $M = d^2$, observe $\frac{N^2 d}{\sqrt{M}} = \frac{N^2 d^2}{M} = N^2$. 
We define $M \geq d^2$ as the \emph{large cache} regime, and $M < d^2$ as the \emph{small cache} regime.
We are primarily interested in the large cache regime, since this is where I/O complexity is sub-quadratic.
This is the regime where FlashAttention outperforms standard implementations of attention, since we can avoid writing the entire $N \times N$ matrix $QK^T$ to memory.

\subsection{Large Cache: \texorpdfstring{$M = \Omega(d^2)$}{}}
\label{sec:standard-large-cache}

For large $M = \Omega(d^2)$, we show that the following result of \cite{DBLP:conf/nips/DaoFERR22} is optimal in terms of I/O complexity.

\begin{theorem}[Theorem 2 of \cite{DBLP:conf/nips/DaoFERR22}]
    \label{thm:flash-attention}
    FlashAttention has I/O complexity $\bigO{\frac{N^2 d^2}{M}}$.
\end{theorem}

To prove a matching lower bound, we bound the number of level-1 vertices in each part of a $M$-partition.
This gives a lower bound on the size of any partition, implying an I/O complexity lower bound via Lemma \ref{lemma:partition-i/o-lb}.

\begin{lemma}
    \label{lemma:part-internal-vertices-ub}
    Suppose $M = \Omega(d^2)$ and let $\partition$ be a $M$-partition of the computational graph in Figure \ref{fig:flash-attention-2-graph}. 
    Let $V' \in \partition$.
    Then, $V'$ contains at most $\bigO{\frac{M^2}{d}}$ level-1 vertices.
\end{lemma}

\begin{proof}
    First, since $V'$ has a dominator set $D'$ of at most $M$, and each summation tree in the computational graph is disjoint, there are at most $M$ level-1 summation trees containing a dominator vertex.
    Next, since $V'$ has at most $M$ minimum vertices, disjointness again implies that at most $M$ level-1 summation trees contain a minimum vertex.
    See Figure \ref{fig:dominator-minimum-examples} for an example of $V'$ containing dominator and minimum vertices.

    \begin{figure}
        \centering
        \begin{subfigure}{0.45\columnwidth}
          \includegraphics[width=0.9\columnwidth]{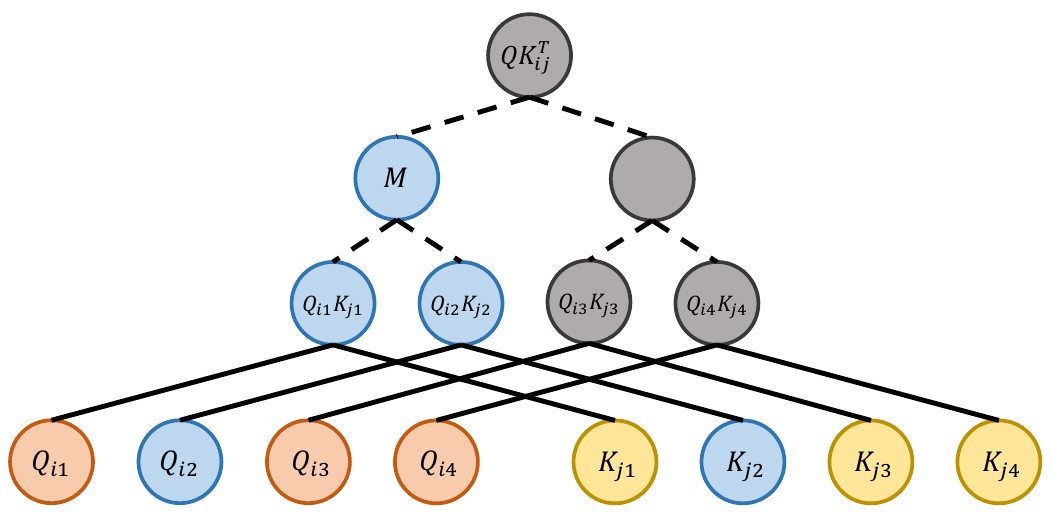}
          \caption{$\set{Q_{i2}, K_{j2}, Q_{i1}K_{j1}}$ are dominator vertices, $\set{M}$ is the minimum vertex.}
        \end{subfigure}
        \begin{subfigure}{0.45\columnwidth}
          \includegraphics[width=0.9\columnwidth]{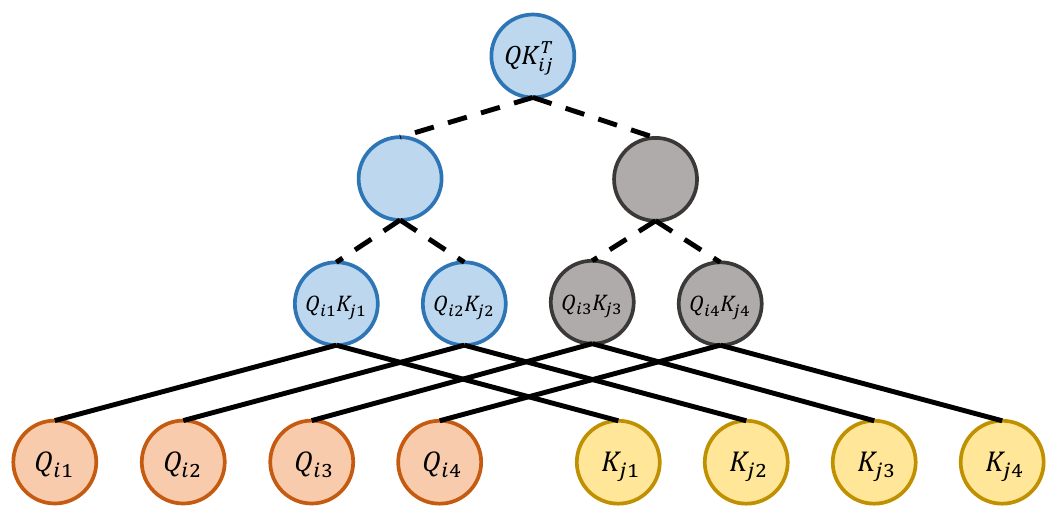}
          \caption{$\set{Q_{i1}K_{j1}, Q_{i2}K_{j2}}$ are dominator vertices, $\set{QK^T_{ij}}$ is the minimum vertex.}
        \end{subfigure}
        \caption{Two examples of summation trees containing dominator and minimum vertices.
        Blue vertices denote elements of the partition $V_i$.}
        \label{fig:dominator-minimum-examples}
    \end{figure}

    Suppose $V'$ contains level-1 vertices not in any of the above $2M$ summation trees (at most $M$ trees containing a dominator and another at most $M$ containing minimum vertices).
    
    Consider a tree $T$ containing level-1 vertices that does not intersect $D'$ or contain a minimum vertex.
    We denote such a tree as \emph{extra}.
    Since $T$ does not contain any minimum vertices, the root $r$ of $T$, representing some entry $(QK^T)[i, j]$, must be in $V'$.
    There are $2d$ input vertices with paths to $r$, namely the inputs $\set{Q[i, \ell], K^T[\ell, j]}_{\ell = 1}^{d}$.
    We denote the $i$-inputs of $T$ as the set $\set{Q[i, \ell]}_{\ell = 1}^{d}$ and the $j$-inputs as the set $\set{K^T[\ell, j]}_{\ell = 1}^{d}$.
    On each path, all vertices but the inputs are in $T$.
    Since $T$ contains no vertices in $D'$, $D'$ must contain all $2d$ input vertices.
    For example, in Figure \ref{fig:summation-tree}, $V'$ contains $(QK^T)[i, j]$ and none of the grey vertices, and therefore must contain all $2d = 8$ input vertices (orange and yellow).

    For two extra trees $T_{ij}, T_{i'j'}$ whose roots correspond to entries $(QK^T)[i, j], (QK^T)[i', j']$ respectively, observe that the $i$-inputs are equal if $i = i'$ and disjoint otherwise.
    Similarly, the $j$-inputs are equal if $j = j'$ and disjoint otherwise.
    Suppose $V'$ contains level-1 vertices in $C$ extra level-1 summation trees.
    The $C$ roots of these trees correspond to $C$ entries in the matrix $QK^T$.
    Suppose the $C$ entries have $I$ distinct values in the row index.
    Choose one entry in $C$ for each row index.
    Each entry requires $d$ $i$-input vertices to be placed in the dominator set, and furthermore these input vertices are disjoint.
    Therefore, there are at least $Id$ input vertices from $Q$ in the dominator set $D'$.
    By a similar argument, if the $C$ entries have $J$ distinct values in the column index, there are at least $Jd$ input vertices from $K$ in the dominator set $D'$.
    In particular, if there are $C$ extra trees, then there are at least $(I + J)d$ input vertices in the dominator set $D'$.
    Finally, we can bound $C$ by observing that $I + J \geq \max(I, J) \geq \sqrt{C}$ since $IJ \geq C$.
    In particular, 
    \begin{equation*}
        \sqrt{C} d \leq (I + J) d \leq |D'| \leq M
    \end{equation*}
    so that $C \leq \frac{M^2}{d^2}$.

    In total, $V'$ contains vertices in at most $C + 2M$ level-1 trees.
    Each level-1 tree contains $O(d)$ level-1 vertices so that $V'$ contains at most,
    \begin{equation*}
        \bigO{\frac{M^2}{d} + M d} = \bigO{\frac{M^2}{d}}
    \end{equation*}
    level-1 vertices, where we have used $M = \Omega(d^2)$.
\end{proof}

Using Lemma \ref{lemma:partition-i/o-lb}, we obtain a lower bound for computing attention using standard matrix multiplication.
In fact, since we have not used any other part of the computational graph, our lower bound holds for any algorithm computing $QK^T$ using standard matrix multiplication.

\begin{restatable}{lemma}{standardMMLargeCacheLB}
    \label{lemma:partition-size-lb-large-mem}
    Suppose $M = \Omega(d^2)$.
    Then, $P(M) = \bigOmega{\frac{N^2 d^2}{M^2}}$ and $Q(M) = \bigOmega{\frac{N^2 d^2}{M}}$.
\end{restatable}

\begin{proof}
    Note that there are $\Omega(N^2 d)$ level-1 vertices.
    Since each part in the $M$-partition contains at most $\bigO{\frac{M^2}{d}}$ level-1 vertices, the lower bound on the number of parts in the $M$-partition follows immediately from Lemma \ref{lemma:part-internal-vertices-ub}.
\end{proof}

\subsection{Small Cache: \texorpdfstring{$M = o(d^2)$}{}}
\label{sec:standard-small-cache}

When $M = o(d^2)$, we show that attention is equivalent to matrix multiplication, establishing a $\bigTheta{\frac{N^2 d}{\sqrt{M}}}$ bound on the I/O complexity.
We first show that there is an algorithm with I/O complexity $\bigO{\frac{N^2 d}{\sqrt{M}}}$.

\begin{theorem}
    \label{thm:small-mem-tiling-alg}
    There is an algorithm computing attention with I/O complexity $\bigO{\frac{N^2 d}{\sqrt{M}}}$.
    Furthermore, this algorithm uses $\bigO{N^2 d}$ time and $\bigO{Nd + N^2}$ space.
\end{theorem}

\paragraph{High Level Overview} Our algorithm follows standard techniques for reducing the I/O complexity of matrix multiplication.
In Phase 1, we compute $A = \exp(QK^T)$ and $D = \diag(A \cdot \mathbf{1})$.
First, we divide $Q, K$ into $\frac{N}{\sqrt{M}} \times \frac{d}{\sqrt{M}}$ blocks of size $\sqrt{M} \times \sqrt{M}$ and proceed to compute $QK^T$ one $\sqrt{M} \times \sqrt{M}$ size block at a time.
In Lines \ref{line:small-cache-mm:i-loop} and \ref{line:small-cache-mm:j-loop}, we iterate over blocks of $QK^T$.
In Lines \ref{line:small-cache-mm:ell-loop} and \ref{line:a-block-summand}, we compute each block of $QK^T$ by summing over $\frac{d}{\sqrt{M}}$ block matrix products.
After computing $QK^T$, we apply exponents entry-wise and sum over rows in Lines \ref{line:small-cache-mm:a-exp} and \ref{line:small-cache-mm:d-summand} to compute the relevant blocks of $A, D$.

In Phase 2, we compute the matrix product $D^{-1} A V$, storing the output in matrix $O$.
We imagine $D^{-1}$ as a vector and partition into $\frac{N}{\sqrt{M}}$ blocks of size $\sqrt{M}$ and partition $A$ into $\frac{N}{\sqrt{M}} \times \frac{N}{\sqrt{M}}$ blocks of size $\sqrt{M} \times \sqrt{M}$.
Lines \ref{line:small-cache-mm:i-loop-2} and \ref{line:small-cache-mm:j-loop-2} iterates over blocks of $O$.
Lines \ref{line:small-cache-mm:k-loop-2} and \ref{line:o-block-summand} compute a $\sqrt{M} \times \sqrt{M}$ block of $O$ by summing over $\frac{N}{\sqrt{M}}$ matrix products and scaling by the approprite block of $D^{-1}$.
This completes the overview of the algorithm.

\begin{algorithm}[H]
\SetKwInOut{Input}{Input}
\SetKwInOut{Output}{Output}
\Input{Matrices $Q, K, V \in \R^{N \times d}$, Cache size $M$}
\Output{$D^{-1}AV$ where $A = \exp(QK^T)$ and $D = \diag(A \cdot \mathbf{1})$}

\BlankLine

$B \gets \floor{\sqrt{M/4}}$

\textcolor{blue}{Phase 1: Compute $D, A$}

\For{$1 \leq i \leq \ceil{N/B}$}{
    \label{line:small-cache-mm:i-loop}
    
    Initialize $d^{(B)}[i] \gets 0^{B}$ in cache
    
    \For{$1 \leq j \leq \ceil{N/B}$}{
        \label{line:small-cache-mm:j-loop}
        
        Initialize $A^{B}[i, j] \gets 0^{B \times B}$ in cache
        
        \For{$1 \leq \ell \leq \ceil{d/B}$}{
            \label{line:small-cache-mm:ell-loop}
            
            Read $Q^{(B)}[i, \ell]$ and $(K^T)^{(B)}[\ell, j]$ into cache

            Compute $A^{B}[i, j] \gets A^{B}[i, j] + Q^{(B)}[i, \ell] (K^T)^{(B)}[\ell, j]$ in cache
            \label{line:a-block-summand}

            Delete $Q^{(B)}[i, \ell]$ and $(K^T)^{(B)}[\ell, j]$ from cache
        }

        Compute $A^{B}[i, j] \gets \exp(A^B[i, j])$ and write $A^{B}[i, j]$ into memory
        \label{line:small-cache-mm:a-exp}

        Compute $d^{(B)}[i] \gets d^{(B)}[i] + A^{B}[i, j] \cdot \mathbf{1}$ in cache
        \label{line:small-cache-mm:d-summand}

        Delete $A^{B}[i, j]$ from cache
    }

    Write $d^{(B)}[i]$ to memory and delete $d^{(B)}[i]$ from cache
}

\textcolor{blue}{Phase 2: Compute $D^{-1} A V$}

\For{$1 \leq i \leq \ceil{N/B}$}{
    \label{line:small-cache-mm:i-loop-2}
    
    Read $d^{(B)}[i]$ into cache
    
    \For{$1 \leq j \leq \ceil{d/B}$}{
        \label{line:small-cache-mm:j-loop-2}
        
        Initialize $O^{(B)}[i, j] \gets 0^{B \times B}$ in cache 
        
        \For{$1 \leq k \leq \ceil{N/B}$}{
            \label{line:small-cache-mm:k-loop-2}
            
            Read $A^{(B)}[i, k]$ and $V^{(B)}[k, j]$ into cache

            Compute $O^{(B)}[i, j] \gets O^{(B)}[i, j] + \diag\left(d^{(B)}[i]\right)^{-1} A^{(B)}[i, k] V^{(B)}[k, j]$
            \label{line:o-block-summand}

            Delete $A^{(B)}[i, k]$ and $V^{(B)}[k, j]$ from cache
        }

        Write $O^{(B)}[i, j]$ to cache and delete $O^{(B)}[i, j]$ from cache
    }
    Delete $d^{(B)}[i]$ from cache
}

\caption{$\squareTiling(Q, K, V, M)$} 
\label{alg:square-tiling-attention}
\end{algorithm}

\begin{proof}[Correctness of Algorithm \ref{alg:square-tiling-attention}]
    We begin with Phase 1, showing that each block $A^{(B)}[i, j]$ is computed correctly. 
    Fix a block $A^{(B)}[i, j]$.
    For any indices $i' \in [(i - 1) B + 1, i B]$ and $j' \in [(j - 1) B + 1, j B]$, 
    \begin{equation*}
        A[i', j'] = \sum_{\ell' = 1}^{d} Q[i, \ell'] K^T[\ell', j] = \sum_{\ell = 1}^{\ceil{d/B}} \sum_{\ell' = (\ell - 1)B + 1}^{\ell B} Q[i, \ell'] K^T[\ell', j] 
    \end{equation*}
    In Line \ref{line:a-block-summand}, we iterate over $i', j', \ell'$, adding $Q[i', \ell']K^T[\ell', j']$ to $A[i', j']$. 
    After iterating over $1 \leq \ell \leq \ceil{d/B}$, $A^{(B)}[i, j]$ is computed so that after applying $\exp$ entry-wise, we write the correct block $A^{(B)}[i, j]$ into memory.
    Next, since $d$ should contain row-sums,
    \begin{equation*}
        d[i'] = \sum_{j' = 1}^{N} A[i', j'] = \sum_{j = 1}^{\ceil{N/B}} \sum_{j' = (j - 1)B + 1}^{jB} A[i', j'] 
    \end{equation*}
    we correctly write $d^{(B)}[i]$ into memory.
    Throughout Phase 1, the number of items in cache is $3B^2 + B \leq 4B^2 \leq M$.

    Now, we proceed to Phase 2. 
    Let $D = \diag(d) = \diag(A \cdot \mathbf{1})$.
    Similarly, for $i' \in [(i - 1)B + 1, iB], j' \in [(j - 1)B + 1, jB]$ 
    \begin{equation*}
        O[i', j'] = \frac{1}{D[i', i']} \sum_{k' = 1}^{N} A[i', k'] V[k', j'] = \sum_{k = 1}^{\ceil{N/B}} \sum_{k' = (k - 1)B + 1}^{kB} \frac{A[i', k'] V[k', j']}{D[i', i']}
    \end{equation*}
    which is computed by the loop in Phase 2.
    In particular, Line \ref{line:o-block-summand} adds the appropriate value for each block $k$.
    The overall size of cache required is $B + 3B^2 \leq 4B^2 \leq M$.
    Thus, Algorithm \ref{alg:square-tiling-attention} correctly computes $O = D^{-1} A V$ with cache size $M$.
\end{proof}

\begin{proof}[I/O Complexity of Algorithm \ref{alg:square-tiling-attention}]
    In Phase 1, for each iteration through $i, j, \ell$, the algorithm reads $O(B^2)$ values from memory into cache.
    This dominates the I/O complexity of the algorithm.
    The I/O complexity of Phase 1 is therefore $\bigO{\frac{N^2 d}{B^3} B^2} = \bigO{\frac{N^2 d}{B}} = \bigO{\frac{N^2 d}{\sqrt{M}}}$.

    Similarly for Phase 2, the I/O complexity is dominated by the reading $A, V$ into the cache and this has I/O complexity $\bigO{\frac{N^2 d}{\sqrt{M}}}$, thus bounding the overall I/O complexity.
\end{proof}

\begin{proof}[Time and Space Complexity of Algorithm \ref{alg:square-tiling-attention}]
    Since we use standard matrix multiplication, the overall time complexity is $O(N^2 d)$.
    The space required is $O(N d + N^2)$ as the algorithm stores matrices $Q, K, V, M$ and the vector $d$.
\end{proof}

We now show that this is tight for $M \leq d^2$.
We proceed by a reduction to the I/O complexity of matrix multiplication, invoking the following result.

\begin{lemma}[Corollary 6.2 of \cite{redblue1981}]
    \label{lemma:matrix-mult-io-lb}
    Let $A \in \R^{m \times k}$ and $B \in \R^{k \times n}$.
    The standard algorithm for matrix multiplication satisfies $Q(M) = \bigOmega{\frac{m k n}{\sqrt{M}}}$.
\end{lemma}

\begin{theorem}
    \label{thm:i/o-lb-small-mem}
    Suppose $M = o(d^2)$.
    Then, the I/O complexity of attention using standard matrix multiplication is at least $\bigOmega{\frac{N^2 d}{\sqrt{M}}}$
\end{theorem}

\begin{proof}
    The lower bound follows from a reduction to matrix multiplication. 
    We take advantage of the fact that if $M \leq d^2$, $\frac{N^2 d}{\sqrt{M}} \geq N^2$, so the algorithm can afford to write the attention matrix $A$ explicitly to memory.
    Given an algorithm $\innerAlg$ for attention, we have the following algorithm for matrix multiplication.
    Given inputs $Q, K$, we execute $\innerAlg$ with one modification: whenever an entry of $QK^T$ is computed for the first time, write this entry to memory.
    Over the course of the algorithm, this computes $QK^T$ and adds at most $N^2$ additional I/Os, which any attention algorithm must use whenever $M < d^2$.
    We give the reduction in terms of the computational graph below.

    Suppose for contradiction there is an algorithm $\innerAlg$ computing attention with I/O complexity $\littleO{\frac{N^2 d}{\sqrt{M}}}$.
    Consider $\innerAlg$ as a complete calculation on the computational graph described in Figure \ref{fig:flash-attention-2-graph}.
    Since $\innerAlg$ is a complete calculation and the set of input and output vertices are disjoint, every single vertex in the graph must have a pebble on it at some configuration in the calculation.
    Consider then the algorithm $\outerAlg$ which executes $\innerAlg$ with the following modifications:
    \begin{enumerate}
        \item Whenever a blue pebble is deleted from a vertex in $QK^T$, do not delete.
        \item Whenever a red pebble is placed on a vertex in $QK^T$ for the first time, place also a blue pebble on this vertex.
    \end{enumerate}
    The two properties guarantee that $\outerAlg$ will have at least the pebbles that $\innerAlg$ has, while any additional pebbles must be blue, so that $\outerAlg$ respects the constraint on the overall number of red pebbles at any given configuration.
    In particular, $\outerAlg$ is a valid calculation that computes $Q K^T$.
    We now analyze the I/O complexity of $\outerAlg$.
    If $Q_{\innerAlg}$ denotes the I/O complexity of $\innerAlg$, the additional writes due to the second rule imply an overall I/O complexity of,
    \begin{equation*}
        Q_{\innerAlg} + N^2 = \littleO{\frac{N^2 d}{\sqrt{M}}}
    \end{equation*}
    Since $\outerAlg$ computes $QK^T$, this contradicts Lemma \ref{lemma:matrix-mult-io-lb}.
\end{proof}

\section{I/O Complexity of Attention with Fast Matrix Multiplication}
\label{sec:i/o-attention-fmm}

We can in fact lower bound the I/O complexity of any algorithm computing attention exactly, including those using fast matrix multiplication. 
The only assumption we require is that the algorithm computes the matrix product $QK^T$ explicitly (whether or not it writes the result to memory).
Since we do not make any further assumptions on the algorithm, the previous approach of analyzing a computational directed acyclic graph is not sufficient \cite{redblue1981}.
Instead, we relate I/O complexity to compression lower bounds.
As discussed previously, we are primarily interested in the large cache regime where $M = \Omega(d^2)$.

\subsection{Large Cache: \texorpdfstring{$M = \Omega(d^2)$}{}}

Given some algorithm $\innerAlg$ and a sample execution, we split this computation into batches of roughly $M$ I/O operations each using the framework of \cite{DBLP:conf/pods/PaghS14}.

\begin{restatable}{lemma}{executionEpochPartition}[Theorem 3 of \cite{DBLP:conf/pods/PaghS14}]
    \label{lemma:execution-epoch-partition}
    Suppose $\innerAlg'$ is an execution of algorithm $\innerAlg$ on a machine with cache of size $M$.
    The execution $\innerAlg'$ can be split into $T$ epochs of at most $M$ I/O operations, such that in each epoch the algorithm $\innerAlg$ has access to a cache of size at most $2M$ and no I/O operations.
    Furthermore, the I/O complexity of $\innerAlg'$ is at least $(T - 1) M$.
\end{restatable}

Intuitively, the algorithm uses the extra $M$ entries in the cache of size $2M$ to create a buffer for the next $M$ I/O operations.

\begin{proof}
    Let $\innerAlg$ be any algorithm computing exact attention and $\innerAlg'$ an arbitrary execution of $\innerAlg$ on machine with a cache size of $M$ bits. 
    Note that given $\innerAlg'$ all the decisions made by algorithm $\innerAlg$ are already taken.

    We proceed to simulate the execution $\innerAlg'$ on a machine with cache size of $2M$ so that the computation is split into epochs and I/O operations are performed only at the start and end of each epoch.
    Split the cache into two pieces, one block of size $M$ to simulate the cache of $\innerAlg$ and one block as a buffer for I/O operations.
    At the start of the epoch, the simulation considers all of the $M$ next I/O operations, performs all read I/Os by filling the buffer.
    During the epoch, any I/O operation is simulated by writing data between the two blocks of cache.
    Finally, at the end of the epoch, the simulation takes all written entries in the buffer and writes them to memory.
    In particular, in each epoch, no I/O operations are performed so that the algorithm only has access to only a cache of size $2M$ with no I/O.

    Finally, in every epoch except for the last, $M$ I/O operations are performed, so the I/O complexity of the execution $\innerAlg'$ is at least $(T - 1) M$.
\end{proof}

To show our lower bound, we will simplify the problem and assume the matrices $Q, K, V$ have entries in some finite field $\F_{q}$.
This is similar to the ``indivisibility" assumptions of \cite{DBLP:conf/dimacs/ArgeM98, DBLP:conf/pods/PaghS14}, as the field size fixes some bound on the amount of information in a single cache entry.
Otherwise, if each cache entry can hold an arbitrary real number, then even when $M = 1$ we could encode the entire matrices $Q, K$ in a single cache entry.

Under this assumption, we will assume the cache holds $M$ elements of $\F_q$, and we will correspondingly define I/O complexity as the number of finite field elements read and written between the memory hierarchy.
In practice, matrices $Q, K$ have arbitrary real entries.
Since our lower bounds only need to consider algorithms computing $QK^T$, we may restrict the inputs to some finite field $\F_q$ without loss of generality (choosing $q$ large enough to avoid overflows under arithmetic operations) and we can assume $q$ is of polynomial size since we do not need to consider the softmax operation.
We will separately consider the cases $q = 2$ (i.e. every entry in the cache and the matrices is a single bit) and the case for an arbitrarily large finite field of size $q$.

\subsection{I/O Complexity and Compression}

We will prove our I/O lower bound by arguing that any algorithm computing entries of $QK^T$ amounts to an efficient compression protocol.
First, we show a lower bound for any compression protocol computing entries of $QK^T$.

\begin{definition}[Matrix Compression]
    \label{def:matrix-entry-compression}
    Let $B \geq 0$.
    In the $B$-entry matrix compression problem $\matrixEntryCompression_B$ Alice is given input matrices $Q, K \in \F_{q}^{N \times d}$.
    Alice must send a message to Bob so that Bob can compute at least $B$ entries in $QK^T$.
\end{definition}

We review the definition of one-way communication complexity below.

\begin{definition}[One-Way Communication Complexity]
    \label{def:one-way-cc}
    Let $f: \domain_A \times \domain_B \rightarrow \range$ be an arbitrary function.
    Suppose Alice has $x \in \domain_A$ and Bob has $y \in \domain_B$.
    A \emph{one-way communication protocol} computing $f$ is a pair of functions $(E, D)$ such that $D(E(x), y) = f(x, y)$ for all $(x, y) \in \domain_A \times \domain_B$.
    The \emph{complexity} of the protocol is $\max_{(x, y)}  |E(x)|$.
    The \emph{one-way communication complexity} of $f$ is the complexity of the optimal protocol. 
    \begin{equation*}
        \oneWayCC(f) = \min_{(E, D)} \max_{(x, y) \in \domain_A \times \domain_B} |E(x)|
    \end{equation*}
\end{definition}

The one-way communication complexity of $\matrixEntryCompression_B$ is $O(B \log q)$, since Alice can compute $QK^T$ and transmit the relevant $B$ entries.
Instead, if Bob computes a square sub-matrix of $QK^T$, Alice can instead send the relevant bits of $Q, K$, transmitting only $O(d \sqrt{B} \log q)$ entries.
We give a lower bound of $\bigOmega{\min(d \sqrt{B} \log q, B \log q)}$, showing that this is essentially tight.

Our key lemma states that any algorithm outputting many bits of $QK^T$ in one epoch (as described in Lemma \ref{lemma:execution-epoch-partition}) gives an efficient compression protocol for $\matrixEntryCompression_B$.

\begin{theorem}
    \label{thm:i/o-compression-protocol}
    Let $\innerAlg'(Q, K)$ be an execution of an algorithm $\innerAlg$ on input $Q, K$ on a machine with cache of size $M$.
    Let $B_t$ be the number of entries of $QK^T$ computed in the $t$-th epoch as described in Lemma \ref{lemma:execution-epoch-partition} and $B'(Q, K) = \max_{t} B_t$.
    Define,
    \begin{equation*}
        B^* = \min_{Q, K} B'(Q, K)
    \end{equation*}
    so that on every input $\innerAlg$ computes at least $B^*$ entries of $QK^T$ on some epoch.
    Then, the one-way communication complexity of $\matrixEntryCompression_{B^*}$ is at most $2M \log q$.
\end{theorem}

\begin{proof}
    We will use $\innerAlg$ to construct a compression protocol.
    Given input matrices $Q, K$, we execute $\innerAlg$ on $Q, K$ to obtain an execution $\innerAlg'$. 
    As described in Lemma \ref{lemma:execution-epoch-partition}, the execution $\innerAlg'$ can be split into epochs, where in each epoch the algorithm $\innerAlg$ has access only to $2M$ finite field elements in cache and reads no other inputs from memory in this epoch.
    Let $t^*$ be an epoch in which the algorithm $\innerAlg$ computes $B'(Q, K)$ entries of the product $QK^T$.
    In particular, Alice can send to Bob the state of the cache of size at most $2M$ at the beginning of the $t^*$-th epoch, so that Bob computes $B'(Q, K) \geq B^*$ entries of $QK^T$.
\end{proof}

\subsection{Matrix Compression Lower Bounds}

In this section, we prove a lower bound on the one-way communication complexity of the $B$-entry $\matrixEntryCompression$ problem.
More precisely, we will show an upper bound on the number of entries that can be computed given a message of size $M$.
Our previous discussion then implies an upper bound on the number of bits computed in each epoch, therefore lower bounding the number of epochs and the I/O complexity.

As a warmup, we give a simple lower bound of $\sqrt{B} \log q$.

\begin{restatable}{lemma}{MaxDimCacheLB}
    \label{lemma:max-dim-cache-lb}
    Let $B \geq 0$.
    Then, the one-way communication complexity of $\matrixEntryCompression_B$ is at least $\sqrt{B} \log q$.
\end{restatable}

\begin{proof}
    Let $I \subset [N]^2$ denote any set of indices of the computed entries of $QK^T$ where $|I| = B$.
    Define $R_I$ to be the distinct row indices in $I$ and $C_I$ to be the distinct column indices in $I$ so that $I \subset R_I \times C_I$.
    
    Without loss of generality, assume $|R_I| \geq |C_I|$.
    We claim there are at least $q^{|R_I|}$ distinct values in the entries of $QK^T$ indexed by $I$.
    In particular, let $Q, K$ both be matrices with non-zero values only in the first column.
    In $K$, we set every entry in the first column to $1$.
    In this case, each row of $QK^T$ will be the same, so we can assume the $B$ entries are $|R_I|$ entries in a column of $QK^T$.
    Since we can arbitrarily set the entries of $Q$ indexed by $R_I$, we can obtain $q^{|R_I|}$ different outputs.
    Since there are at least $q^{|R_I|}$ outputs, the message length $M$ must be at least $|R_I| \log q = \max(|R_I|, |C_I|) \log q$ in order to unambiguously determine the correct output.
    Then,
    \begin{equation*}
        B \leq R_I C_I \leq \max(R_I, C_I)^2 \leq \left(\frac{M}{\log q}\right)^2
    \end{equation*}
    Thus, $M \geq \sqrt{B} \log q$.
\end{proof}

We generalize this for matrices $Q, K$ of dimension $N \times d$ with rank $d \geq 1$.

\begin{lemma}
    \label{lemma:d-rank-cache-lb}
    Suppose $Q, K \in \F_{q}^{N \times d}$ with finite field $\F_{q}$ of size $q > N$.
    Then, the one-way communication complexity of $\matrixEntryCompression_B$ is at least $\min \left( d \sqrt{\frac{B}{2}}, \frac{B}{4} \right) \log q$.
\end{lemma}

We show that for any set of $B$ indices, there are more than $q^M$ possible values. Thus, a message length of at least $M \log q$ is required to specify these entries exactly.

\begin{proof}
    Let $M$ denote the maximum message length in the communication protocol.
    Let $I = [N]^2$ denote the indices of $QK^T$ computed with $B = |I|$.
    Define $R_I$ to be the distinct row indices in $I$ and $C_I$ to be the distinct column indices in $I$ so that $I \subset R_I \times C_I$. 
    
    For each $i \in R_I$, let $R_i = \set{j \given (i, j) \in I}$ be the computed entries in the $i$-th row of $QK^T$.
    Similarly, for each $j \in C_I$, let $C_j = \set{i \given (i, j) \in I}$ be the computed entries in the $j$-th column of $QK^T$.
    Next, define $L_R = \set{i \given |R_i| \geq d}$ and $L_C = \set{j \given |C_j| \geq d}$.
    We also define $S_R = \set{(i, j) \in I \given i \not\in L_R}$ and $S_C = \set{(i, j) \in I \given j \not\in L_C}$.

    First, suppose $\max(|L_R|, |L_C|) \geq \sqrt{B/2}$.
    Without loss of generality, assume $|L_R| \geq \sqrt{B/2}$.
    Since $q > N$, we fix $K$ to be the Vandermonde matrix guaranteed by Lemma \ref{lemma:vandermonde-matrix}.
    In particular, every subset of $d$ columns in $K^T$ is linearly independent.
    
    Fix some row $i \in R_I$.
    We claim that there are at least $q^{\min(|R_i|, d)}$ distinct values in the $R_i$ indices of the $i$-th row.
    First, suppose $|R_i| \leq d$.
    By the construction of matrix $K$ and $|R_i| \leq d$, the columns of $K^T$ indexed by $R_i$ are linearly independent.
    Then, for any $\vec{v} \in \F_q^{|R_i|}$, we can set the $i$-th row of $Q$ to be,
    \begin{equation*}
        Q[i] = \vec{v} \begin{pmatrix}
            K[r_1]^T K[r_2]^T \dotsc K[r_{|R_i|}]^T
        \end{pmatrix}^{-1}
    \end{equation*}
    where $K[r_i]$ is the $r_i$-th row of $K$ and $R_i = \set{r_1, \dotsc, r_{|R_i|}}$.
    In particular, this choice of $Q[i]$ ensures that the $R_i$ entries of $QK^T$ are exactly $\vec{v}$ so that there are at least $q^{|R_i|}$ distinct values.
    Whenever $|R_i| > d$, we simply take an arbitrary subset of $R_i$ of size $d$, and use their linear independence to proceed with the same argument, thus obtaining the lower bound of $q^{\min(|R_i|, d)}$.

    Finally, note that the we can obtain these distinct values for each row in $R_I$ independently, since we have fixed $K$ as the Vandermonde matrix and for each row we only modify the entries of $Q[i]$.
    In particular, the total number of possible distinct values in all the entries of $I$ is at least,
    \begin{equation*}
        \prod_{i \in R_I} q^{\min(|R_i|, d)}
    \end{equation*}
    so that we obtain the following lower bound on the message length of the communication protocol in order to unambiguously specify $B$ entries,
    \begin{align*}
        M &\geq \left( \sum_{i \in R_I} \min(|R_i|, d) \right) \log q \\
        &\geq \left( d |L_R| + |S_R| \right) \log q \\
        &\geq d \sqrt{\frac{B}{2}} \log q \numberthis \label{eq:d-rank-m-lb}
    \end{align*} 
    where $S_R = \set{(i, j) \in I \given i \not\in L_R}$ and the final inequality follows from our assumption on $L_R$.
    In particular, this implies $M \geq d \sqrt{B/2} \log q$ as desired.

    Finally, we consider the case $\max(|L_R|, |L_C|) < \sqrt{B/2}$.
    Then, the number of entries in $L_R \times L_C$ is at most $\frac{B}{2}$.
    Note that any entry of $I$ not in $L_R \times L_C$ must be in $S_R \cup S_C$ and therefore,
    \begin{equation*}
        \frac{B}{2} \leq |S_R| + |S_C|
    \end{equation*}
    Assume without loss of generality $|S_R| \geq |S_C|$.
    Equation \ref{eq:d-rank-m-lb} then implies $M \geq \frac{B}{4} \log q$.
\end{proof}

In our lower bound construction, we require matrices satisfying strong linear independence constraints.
Specifically, we require a $N \times d$ matrix such that every subset of $d$ rows is linearly independent.
When the matrices have elements in a large finite field, this is obtained by Vandermonde matrices.

\begin{restatable}{lemma}{vandermonde}
    \label{lemma:vandermonde-matrix}
    There is a $N \times d$ Vandermonde matrix with entries in a finite field $\F_{q}$ of size $q > N$ where every subset of $d$ rows is linearly independent.
\end{restatable}

\begin{proof}
    Consider the $N \times d$ Vandermonde matrix,
    \begin{equation*}
        V = \begin{pmatrix}
            1 & \alpha_1 & \alpha_1^2 & \hdots & \alpha_1^{d - 1} \\
            1 & \alpha_2 & \alpha_2^2 & \hdots & \alpha_2^{d - 1} \\
            \vdots & \vdots & \vdots & \ddots & \vdots \\
            1 & \alpha_N & \alpha_N^2 & \hdots & \alpha_N^{d - 1} \\
        \end{pmatrix}
    \end{equation*}
    where $\alpha_1, \alpha_2, \dotsc, \alpha_{N}$ are distinct elements in $\F_{q}$, since we choose $q > N$.
    Then, for any subset of $d$ rows, the determinant of this sub-matrix is,
    \begin{equation*}
        0 \neq \prod_{1 \leq i < j \leq d} (\alpha_{k_i} - \alpha_{k_j})
    \end{equation*}
    where $k_1, \dotsc, k_d$ are the indices of the $d$ rows.
    Since the determinant of this matrix is non-zero, the rows are linearly independent.
\end{proof}

We are now ready to prove the main result of this section.
For $M \geq d^2$, this matches the upper bound given by \cite{DBLP:conf/nips/DaoFERR22}.

\largeFieldAttention

\begin{proof}
    The theorem follows from combining Theorem \ref{thm:i/o-compression-protocol} and Lemmas \ref{lemma:execution-epoch-partition} and \ref{lemma:d-rank-cache-lb}.
    Consider an arbitrary algorithm $\innerAlg$ and its best execution $\innerAlg'$, on an input $Q, K, V$ with $B'(Q, K) = B^*$ as described in Theorem \ref{thm:i/o-compression-protocol}.

    Since $q > N$ we apply Lemmas \ref{lemma:execution-epoch-partition} and \ref{lemma:d-rank-cache-lb} and obtain
    \begin{equation*}
        \min \left( d \sqrt{\frac{B^*}{2}}, \frac{B^*}{4} \right) \log q \leq 2M \log q
    \end{equation*}

    In particular, the maximum number of entries of $QK^T$ computed in any epoch has the upper bound,
    \begin{equation*}
        B^* = \bigO{\max \left(\frac{M^2}{d^2}, M \right)}
    \end{equation*}
    
    Then since the algorithm computes all $N^2$ values of $QK^T$ (regardless of whether these values are written to memory from cache), the number of epochs is at least,
    \begin{equation*}
        T = \bigOmega{\min\left( \frac{N^2 d^2}{M^2}, \frac{N^2}{M}\right)}
    \end{equation*}
    Then, since the I/O complexity of $\innerAlg$ is at least $(T - 1) M$, this completes the lower bound.
\end{proof}

Whenever $M \geq d^2$, we obtain a lower bound matching the $\bigO{\frac{N^2 d^2}{M}}$ algorithm of \cite{DBLP:conf/nips/DaoFERR22}.

\subsection{Binary Matrix Compression Lower Bounds}

In the previous section, we obtained a tight lower bound for the I/O complexity of attention when the entries are allowed to come from a large finite field.
We believe it is an interesting theoretical question to investigate the I/O complexity with entries in smaller finite fields.
Specifically, we will consider the binary finite field $\F_2 = \set{0, 1}$.

The assumption $q > N$ was only required to construct a matrix satisfying strong linear independence constraints in Lemma \ref{lemma:vandermonde-matrix}.
Even relaxing this constraint and considering $q = 2$, we can still construct matrices satisfying fairly strong linear independence constraints using error correcting codes.
In fact, our previous use of Vandermonde matrices can be interpreted as using Reed-Solomon codes by allowing for arbitrarily large finite fields.
Recall that a linear code $\class \subset \set{0, 1}^{N}$ is the null-space of the parity check matrix $H$. 
Using Binary BCH codes, we can obtain a similar result with binary matrices.

\begin{restatable}{lemma}{binaryLinIndepMatrix}
    \label{lemma:binary-lin-indep-matrix}
    There exists a matrix $K \in \set{0, 1}^{N \times d}$ such that every set of $\frac{2 d}{\log(N + 1)} - 1$ rows is linearly independent.
\end{restatable}

Before proving Lemma \ref{lemma:binary-lin-indep-matrix}, we require several intermediate results.
First, we state the following standard lemma relating distance of linear codes to linear independence in the parity check matrix.

\begin{restatable}{lemma}{distanceLinIndepEquiv}
    \label{lemma:distance-lin-indep-equivalence}
    A linear code has distance $d$, if and only if any $(d - 1)$ columns of the parity check matrix is linearly independent and there exist $d$ columns that are linearly dependent.
\end{restatable}

\begin{proof}
    Consider a code $\class$ of length $n$, dimension $k$, and distance $d$.
    Suppose there is a set of $(d - 1)$ linearly dependent columns.
    Then, there is a vector $x$ with $\wt(x) \leq d - 1$ such that $Hx = 0$, contradicting the minimum distance $d$ of $\class$.
    Since the distance of the code is $d$, there exists a vector $x$ such that $Hx = 0$ and $\wt(x) = d$.
    In particular, there exists a subset of $d$ independent columns.
    
    To prove the converse, note that the conditions on $H$ imply there exists a codeword of weight $d$ and no codeword of weight less than $d$, so that the minimum distance of the code is exactly $d$.
\end{proof}

Next, we use the fact that binary BCH codes are optimal high rate codes.
Recall that a code with parity check matrix $H$ of dimension $d \times N$ has dimension at least $N - d$.

\begin{restatable}{lemma}{binaryBCHCodes}{\cite{hocquenghem1959codes, DBLP:journals/iandc/BoseR60a}}
    \label{lemma:binary-bch-codes}
    For a length $N = 2^{m} - 1$ and a distance $s$, there exists a code $\bch{N}{s}$ with dimension at least $N - \ceil{\frac{s - 1}{2}} \log (N + 1)$.
\end{restatable}

We provide the definition of BCH codes.
Recall an element $\alpha \in \F$ in a finite field is a \emph{primitive} element if it generates the multiplicative group $\F^*$.

\begin{definition}[\cite{hocquenghem1959codes, DBLP:journals/iandc/BoseR60a}]
    \label{def:binary-bch-codes}
    For length $N = 2^{m} - 1$, distance $s$, and primitive element $\alpha \in \F_{2^m}^*$, the binary BCH code is defined,
    \begin{equation*}
        \bch{N}{s} = \set{(c_0, \dotsc, c_{N - 1}) \given c(\alpha) = \dotsc = c(\alpha^{s - 1}) = 0}
    \end{equation*}
    where $c(X) = c_0 + c_1 X + \dotsc + c_{N - 1} X^{N - 1}$.
\end{definition}

The proof of Lemma \ref{lemma:binary-bch-codes} is a standard exercise.

\begin{proof}[Proof of Lemma \ref{lemma:binary-bch-codes}]
    To show a lower bound on the dimension of the code, we argue that the parity check matrix $H$ does not have too many rows.
    We begin with a weaker bound of $N - (s - 1) \log (N + 1)$.
    In particular, we show that each constraint $c$ can be written as $m = \log(N + 1)$ linear constraints.
    
    We choose a basis $\beta = \set{\beta_1 = 1, \beta_2, \dotsc, \beta_{m}}$ of $\F_2^{m}$ as a vector space.
    For any $x \in \F_2^{m}$, consider the linear map $x \mapsto \alpha x$ which can be written as $x \mapsto M_{\alpha} x$ for some matrix $M_{\alpha} \in \F_2^{m \times m}$ where $x$ is represented in the above basis.
    Then, the constraint $c(\alpha) = 0$ can be viewed as,
    \begin{equation*}
        \begin{pmatrix} c_0 \\ 0 \\ \vdots \\ 0 \end{pmatrix} + 
        M_{\alpha} \begin{pmatrix} c_1 \\ 0 \\ \vdots \\ 0 \end{pmatrix} + 
        \dotsc +
        M_{\alpha^{N - 1}} \begin{pmatrix} c_{N - 1} \\ 0 \\ \vdots \\ 0 \end{pmatrix} = 
        \begin{pmatrix} 0 \\ 0 \\ \vdots \\ 0 \end{pmatrix}
    \end{equation*}
    Thus, each constraint $c(\alpha_{i}) = 0$ can be viewed as $\log(N + 1)$ linear constraints.
    Since there are $s - 1$ such constraints, this ensures the parity check matrix has dimension $(s - 1) \log(N + 1) \times N$.

    We now prove the improved bound.
    This follows from the fact that $c(\gamma) = 0$ if and only if $c(\gamma^2) = 0$.
    In particular, $\floor{\frac{s - 1}{2}}$ constraints are redundant, leaving only $\ceil{\frac{s - 1}{2}}$ relevant constraints.

    It remains to show $c(\gamma^2) = 0$ if and only if $c(\gamma) = 0$.
    Note that $c(\gamma) = 0$ if and only if $c(\gamma)^2 = 0$.
    Furthermore, for $\alpha, \beta \in \F_2^{m}$, $(\alpha + \beta)^2 = \alpha^2 + \beta^2$.
    Then,
    \begin{align*}
        0 &= c(\gamma) = c(\gamma)^2 \\
        &= c_0^2 + (c_1 \gamma)^2 + \dotsc + (c_{N - 1} \gamma^{N - 1})^{2} \\
        &= c_0 + c_1 \gamma^2 + \dotsc c_{N - 1} \gamma^{2(N - 1)} = c(\gamma^2)
    \end{align*}
    where we have used $c_i = c_i^2$ for all coefficients $c_i \in \F_2$.
\end{proof}

We now prove Lemma \ref{lemma:binary-lin-indep-matrix} and describe the construction of the desired matrix $K$ satisfying strong linear independence constraints.

\begin{proof}
    We assume without loss of generality that $N = 2^{m} - 1$ for some $m$.
    If not, we at most double $N$ by choosing the minimum $m$ such that $2^{m} - 1 \geq N$ and take any $N$-row sub-matrix.
    
    From the construction of the code $\bch{N}{s}$, we observe that it has a parity check matrix $H$ of dimension $\ceil{\frac{s - 1}{2}} \log(N + 1) \times N$.
    Since the code has distance $s$, from Lemma \ref{lemma:binary-lin-indep-matrix}, any set of $s - 1$ rows is linearly independent.
    In particular, if $d = \ceil{\frac{s - 1}{2}} \log(N + 1)$, we have,
    \begin{align*}
        s &\geq \frac{2 d}{\log(N + 1)}
    \end{align*}
    giving the desired bound on $s - 1$.
    Thus, we choose $K = H^T$.
\end{proof}

Given the construction of the matrix $K$, we can prove the following analogues of Lemma \ref{lemma:d-rank-cache-lb} and Theorem \ref{thm:attention-i/o-lb}.
These results match the upper bound of Theorem \ref{thm:flash-attention} up to a $O(\log^2 N)$ factor.

\begin{restatable}{lemma}{dRankCacheLBBinary}
    \label{lemma:d-rank-cache-lb-binary}
    Suppose $Q, K$ are binary $N \times d$ matrices. 
    Then, the one-way communication complexity of $\matrixEntryCompression_B$ is at least $\bigOmega{\min \left( \frac{d \sqrt{B}}{\log N}, B \right)}$.
\end{restatable}

\begin{proof}
    The proof follows Lemma \ref{lemma:d-rank-cache-lb} closely, so we only point out the necessary modifications.
    Again let $M$ denote the maximum message length in the communication protocol and $I = [N]^2$ denote the indices of $QK^T$ computed with $B = |I|$.
    Define $R_I, C_I, \set{R_i}_{i = 1}^{N}, \set{C_j}_{j = 1}^{N}$ as in Lemma \ref{lemma:d-rank-cache-lb}.
    
    We modify $L_R = \set{i \given |R_i| \geq \frac{2 d}{\log(N + 1)} - 1}$ and $L_C = \set{j \given |C_j| \geq \frac{2 d}{\log(N + 1)} - 1}$ and define $S_R$ and $S_C$ as before.

    We again consider first the case $\max(|L_R|, |L_C|) \geq \sqrt{B/2}$ and assume $|L_R| \geq \sqrt{B/2}$.
    Instead of the Vandermonde matrix, we fix $K$ to be the matrix guaranteed by Lemma \ref{lemma:binary-lin-indep-matrix}.
    In particular, every subset of $\frac{2 d}{\log(N + 1)} - 1$ columns in $K^T$ is linearly independent.
    
    Following an analogous argument as Lemma \ref{lemma:d-rank-cache-lb}, we obtain the following lower bound on the message length of the communication protocol in order to unambiguously specify $B$ entries,
    \begin{align*}
        M &\geq \sum_{i \in R_I} \min\left(|R_i|, \frac{2 d}{\log(N + 1)} - 1 \right) \\
        &\geq \left( \frac{2 d}{\log(N + 1)} - 1 \right) |L_R| + |S_R| \\
        &= \bigOmega{\frac{d \sqrt{B}}{\log N}} \numberthis \label{eq:d-rank-m-lb-binary}
    \end{align*} 
    where the final inequality follows from our assumption on $L_R$.

    Finally, we consider the case $\max(|L_R|, |L_C|) < \sqrt{B/2}$.
    Again following similar arguments as Lemma \ref{lemma:d-rank-cache-lb} and Equation \ref{eq:d-rank-m-lb-binary}, we have,
    \begin{equation*}
        M \geq \max(|S_R|, |S_C|) \geq \frac{B}{4}
    \end{equation*}
    so that $M$ is at least the minimum of the two bounds.
\end{proof}

We now state the I/O complexity lower bound of attention given binary input matrices.

\binaryAttention

\begin{proof}
    Following similar arguments as Theorem \ref{thm:attention-i/o-lb}, we obtain,
    \begin{equation*}
        \bigOmega{\min \left( \frac{d \sqrt{B^*}}{\log N}, B^* \right)} \leq 2M
    \end{equation*}
    where $B^*$ is as defined in Theorem \ref{thm:i/o-compression-protocol}.
    In particular, the maximum number of entries of $QK^T$ computed in any epoch is at most,
    \begin{equation*}
        B^* = \bigO{\max \left(\frac{M^2 \log^2 N}{d^2}, M \right)}
    \end{equation*}
    
    which gives the I/O complexity lower bound,
    \begin{equation*}
        \bigOmega{\min\left( \frac{N^2 d^2}{M \log^2 N}, N^2 \right)}
    \end{equation*}
    as desired.
\end{proof}

\subsection{Small Cache: \texorpdfstring{$M = o(d^2)$}{}}

In the small cache setting, we proved an equivalence between attention and matrix multiplication in the setting where both are computed using the standard algorithm.
We do the same for algorithms using fast matrix multiplication.

Let $Q_{Att}(M)$ denote the I/O complexity of attention on a machine with cache size $M$.
Let $Q_{\mechanism(a, b, c)}(M)$ denote the I/O complexity of multiplying a $a \times b$ matrix with a $b \times c$ matrix on a machine with cache size $M$.
First, we show matrix multiplication is more expensive than attention.

\begin{restatable}{lemma}{matMultGeqAtt}
    \label{lemma:mat-mult-higher-i/o}
    For all $M$, 
    \begin{equation*}
        Q_{Att}(M) = \bigO{Q_{\mechanism(N, d, N)}(M) + Q_{\mechanism(N, N, d)}(M)}
    \end{equation*}
\end{restatable}

\begin{proof}
    First, we use the given algorithm for rectangular matrix multiplication to compute $QK^T$ with $Q_{\mechanism(N, d, N)}$ I/O operations.
    Then, note that computing $\softmax(QK^T)$ can be done in $O(N^2)$ time and therefore $O(N^2)$ I/O complexity as each operation can be performed with $O(1)$ I/O operations.
    Finally, we use the given algorithm to compute $\softmax(QK^T)V$ with $Q_{\mechanism(N, N, d)}$.
    Then, the overall I/O complexity is,
    \begin{align*}
        Q_{Att}(M) &= \bigO{Q_{\mechanism(N, d, N)} + N^2 + Q_{\mechanism(N, N, d)}} \\
        &= \bigO{Q_{\mechanism(N, d, N)} + Q_{\mechanism(N, N, d)}}
    \end{align*}
    since the I/O complexity of both matrix products is at least $N^2$, as either the input or output has size $N^2$.
\end{proof}

When $M \leq d^2$, the two are equivalent.

\begin{restatable}{lemma}{attGeqMatMult}
    \label{lemma:small-cache-attention-mat-mult}
    Let $M \leq d^2$. 
    Then, $Q_{Att}(M) = \bigOmega{Q_{\mechanism(N, d, N)}(M)}$.
\end{restatable}

\begin{proof}
    First, consider any two input matrices $Q, K^T$.
    We simulate the attention algorithm with the following modification: whenever an entry $(QK^T)_{ij}$ is computed for the first time, we write this entry to memory.
    Our modified algorithm successfully computes the matrix product $QK^T$ using at most $\bigO{Q_{Att}(M) + N^2}$ I/O operations.
    From Theorem \ref{thm:attention-i/o-lb}, we have that whenever $M \leq d^2$, $Q_{Att}(M) = \Omega(N^2)$.
    As a result, we compute Attention with $\bigO{Q_{Att}(M)}$ I/O complexity.
\end{proof}

As a corollary, we obtain the following equivalence in the small cache setting.

\begin{theorem}
    \label{thm:small-cache-fmm-equiv}
    Let $M \leq d^2$.
    Then, 
    \begin{align*}
        Q_{Att}(M) &= \bigOmega{Q_{\mechanism(N, d, N)}(M)} \\
        Q_{Att}(M) &= \bigO{Q_{\mechanism(N, d, N)}(M) + Q_{\mechanism(N, N, d)}(M)}
    \end{align*}
\end{theorem}

\section{Conclusion}

We have established tight I/O complexity lower bounds for attention.
Our lower bound in fact holds for any algorithm computing matrix product $QK^T$.
We give a tight characterization of the I/O complexity of algorithms using standard matrix multiplication, answering an open question of \cite{DBLP:conf/nips/DaoFERR22}.

Furthermore, in the regime of practical interest, where cache size $M \geq d^2$ is large, we extend our lower bound to algorithms using fast matrix multiplication, showing that FlashAttention is optimal even when fast matrix multiplication is allowed.
Furthermore, we establish a connection between communication complexity and I/O complexity, which may be of independent theoretical interest.

We leave the problem of establishing tight I/O complexity bounds for attention in the small cache regime ($M \leq d^2$) (equivalently rectangular matrix multiplication) as an interesting open problem.

\section*{Acknowledgements}

We would like to thank Arya Mazumdar and Harry Sha for helpful discussions.

\bibliographystyle{alpha}
\bibliography{references}

\end{document}